\documentclass[12pt]{article}
\usepackage{graphicx} 
\usepackage{amsmath, amssymb, amsthm,wasysym}
\usepackage{algorithm}
\usepackage{algorithmic}
\usepackage{color}
\usepackage{times}
\usepackage{times}

\newcommand{\figscale}[2]{\includegraphics[scale=#2,clip=false]{#1}}

\newcommand{\field}[1]{\mathbb{#1}}
\newcommand{\R}{\field{R}}
\newcommand{\E}{\field{E}}
\newcommand{\Ind}[1]{\field{I}\bigl\{{#1}\bigr\}}

\renewcommand{\ss}{\subseteq}

\newcommand{\norm}[1]{\left\|{#1}\right\|}

\newcommand{\diag}{\mathrm{Diag}}

\newcommand{\bx}{\boldsymbol{x}}

\newcommand{\bv}{\boldsymbol{v}}

\newcommand{\sL}{{L_s}}

\newcommand{\spin}{\{-1,+1\}}

\newcommand{\scC}{\mathcal{C}}
\newcommand{\cyc}{\mathrm{Cyc}}
\newcommand{\cut}{\mathrm{Cut}}
\newcommand{\scO}{\mathcal{O}}
\newcommand{\scE}{\mathcal{E}}
\newcommand{\tp}{\textsc{TreePartition}}
\newcommand{\ep}{\textsc{EdgePartition}}

\newcommand{\path}{\mathrm{Path}}

\renewcommand{\Pr}{\field{P}}

\newcommand{\scP}{\mathcal{P}}

\newcommand{\hf}{\widehat{f}}

\newcommand{\lambdamin}{\lambda_{\mathrm{min}}}

\newcommand{\Eflip}{ E_{\mathrm{flip}} }

\newcommand{\alg}{\textsc{scccc}}
\newcommand{\algb}{\textsc{cccc}}

\newtheorem{theorem}{Theorem}   
\newtheorem{lemma}[theorem]{Lemma}

\newtheorem{remark}{Remark}

\newtheorem{proposition}{Proposition}

\title{{\bf A Correlation Clustering Approach to Link Classification in Signed Networks\\
-- Full Version --}}

\author{
Nicol\`o Cesa-Bianchi\\ 
Dipartimento di Informatica, Universit\`a degli Studi di Milano, Italy\\
\texttt{nicolo.cesa-bianchi@unimi.it}
\and
Claudio Gentile\\ 
DiSTA, Universit\`a dell'Insubria, Italy\\
\texttt{claudio.gentile@uninsubria.it}
\and
Fabio Vitale\\ 
Dipartimento di Informatica, Universit\`a degli Studi di Milano, Italy\\
\texttt{fabio.vitale@unimi.it}
\and
Giovanni Zappella\\
Dipartimento di Matematica, Universit\`a degli Studi di Milano, Italy\\
\texttt{giovanni.zappella@unimi.it}
}

\begin{document}

\maketitle

\begin{abstract}
Motivated by social balance theory, we develop a theory of link classification in signed networks using the correlation clustering index as measure of label regularity. We derive learning bounds in terms of correlation clustering within three fundamental transductive learning settings: online, batch and active. Our main algorithmic contribution is in the active setting, where we introduce a new family of efficient link classifiers based on 
covering the input graph with small circuits. These are the first active algorithms for link classification with mistake bounds that hold for arbitrary signed networks.
\end{abstract}


\section{Introduction}
Predictive analysis of networked data ---such as the Web, online social networks, or biological networks--- is a vast and rapidly growing research area whose applications include spam detection, product recommendation, link analysis, and gene function prediction. Networked data are typically viewed as graphs, where the presence of an edge reflects a form of semantic similarity between the data associated with the incident nodes. Recently, a number of papers have started investigating networks where links may also represent a negative relationship. For instance, disapproval or distrust in social networks, negative endorsements on the Web, or inhibitory interactions in biological networks. Concrete examples from the domain of social networks and e-commerce are Slashdot, where users can tag other
users as friends or foes, Epinions, where users can give positive or negative ratings not only to products, but also to other users, and Ebay, where users develop trust and distrust towards agents operating in the network. Another example is the social network of Wikipedia administrators, where votes cast by an admin in favor or against the promotion of another admin can be viewed as positive or negative links.
The emergence of signed networks has attracted attention towards the problem of edge sign prediction or \textsl{link classification}. This is the task of determining whether a given relationship between two nodes is positive or negative. In social networks, link classification may serve the purpose of inferring the sentiment between two individuals, an information which can be used, for instance, by recommender systems.

Early studies of signed networks date back to the Fifties. For example, \cite{ha53} and~\cite{ch56} model dislike and distrust relationships among individuals as negatively weighted edges in a graph.
The conceptual context is provided by the theory of {\em social balance}, formulated as a way to understand the origin and the structure of conflicts in a network of individuals whose mutual relationships can be classified as friendship or hostility \cite{hei46}.
The advent of online social networks has witnessed a renewed interest
in such theories, and has recently spurred a significant amount of work ---see, e.g., \cite{GKRT04,KLB09,LHK10b,cntd11,fia11}, and references therein.
According to social balance theory, the regularity of the network depends on the presence of ``contradictory'' cycles. The number of such bad cycles is tightly connected to the correlation clustering index of \cite{BBC04}. This index is defined as the smallest number of sign violations that can be obtained by clustering the nodes of a signed graph in all possible ways. A sign violation is created when the incident nodes of a negative edge belong to the same cluster, or when the incident nodes of a positive edge belong to different clusters. Finding the clustering with the least number of violations is known to be NP-hard \cite{BBC04}.

In this paper, we use the correlation clustering index as a learning bias for the problem of link classification in signed networks. As opposed to the experimental nature of many of the works that deal with link classification
in signed networks, we study the problem from a learning-theoretic standpoint.
We show that the correlation clustering index characterizes the prediction complexity of link classification in three different supervised transductive learning settings.
In online learning, the optimal mistake bound (to within logarithmic factors) is attained by Weighted Majority run over a pool of instances of the Halving algorithm. We also show that this approach cannot be implemented efficiently under standard complexity-theoretic assumptions. In the batch (i.e., train/test) setting, we use standard uniform convergence results for transductive learning \cite{EP09} to show that the risk of the empirical risk minimizer is controlled by the correlation clustering index. We then observe that known efficient approximations to the optimal clustering can be used to obtain polynomial-time (though not practical) link classification algorithms.
In view of obtaining a practical and accurate learning algorithm, we then focus our attention to the notion of two-correlation clustering derived from the original formulation of structural balance due to Cartwright and Harary. This kind of social balance, based on the observation that in many social contexts ``the enemy of my enemy is my friend'', is used by known efficient and accurate heuristics for link classification, like the least eigenvalue of the signed Laplacian and its variants \cite{KLB09}. The two-correlation clustering index is still hard to compute, but the task of designing good link classifiers sightly simplifies due to the stronger notion of bias. In the active learning protocol, we show that the two-correlation clustering index bounds from below the test error of any active learner on any signed graph. Then, we introduce the first efficient active learner for link classification with performance guarantees (in terms of two-correlation clustering) for any signed graph. Our active learner receives a query budget as input parameter, requires time $\scO\bigl(|E|\sqrt{|V|}\ln|V|\bigr)$ to predict the edges of any graph $G=(V,E)$, and is relatively easy to implement.


\section{Preliminaries}\label{s:prel}
We consider undirected graphs $G = (V,E)$ with unknown edge labeling $Y_{i,j} \in \spin$ for each $(i,j) \in E$. 
Edge labels of the graph are collectively represented by the associated signed adjacency matrix $Y$, 
where $Y_{i,j}=0$ whenever $(i,j) \not\in E$. The edge-labeled graph $G$ will henceforth be 
denoted by $(G,Y)$.
Given $(G,Y)$, the cost of a partition of $V$ into clusters is the number of 
negatively-labeled within-cluster edges plus the number of positively-labeled between-cluster edges. 
%
\begin{figure}[t!]
\begin{center}
\figscale{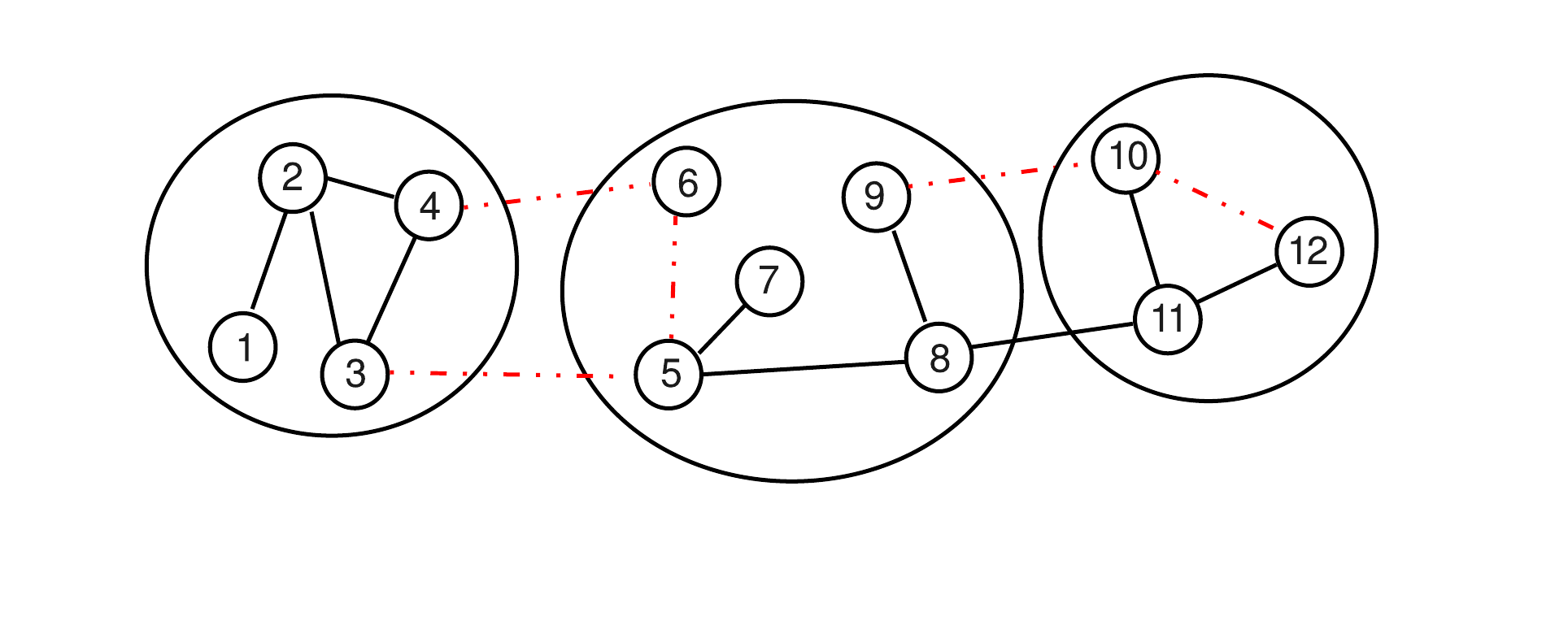}{0.6}
\end{center}
\caption{
\label{f:1}
An edge-labeled undirected graph $(G,Y)$ whose nodes $V = \{1, \ldots, 12 \}$ 
are partitioned in three clusters.
Black solid edges are positive, red dashed edges are negative. 
The cost of the displayed partition is three due to edges $(5,6)$, $(8,11)$, and $(10,12)$. 
Hence $\Delta(Y) \leq 3$. 
There are two bad cycles: $8-9-10-11$ and $10-11-12$, 
hence $\Delta(Y)$ cannot be zero. In fact, in order to remove all 
bad cycles, it suffices to delete edge $(10,11)$. By Proposition~\ref{fa:1}, 
this yields $\Delta(Y) = 1$. A corresponding optimal partition of $V$ is $\bigl\{\{1,2,3,4\}, \{6\}, \{5,7,8,9,11,12\}, \{10\} \bigr\}$.
}
\end{figure}
%
%
We measure the regularity of an edge labeling $Y$ of $G$ through the 
\textsl{correlation clustering index} $\Delta(Y)$. This is defined
as the minimum over the costs of all partitions of $V$. Since the cost of a given partition of $V$
is an obvious quantification of the consistency of the associated clustering, $\Delta(Y)$ quantifies the
cost of the best way of partitioning the nodes in $V$. Note that the number of clusters is not
fixed ahead of time.
In the next section, we relate this regularity measure to the optimal number of prediction mistakes in edge classification problems.


A \textsl{bad cycle} in $(G,Y)$ is a simple cycle (i.e., a cycle with no repeated nodes, except the first one) 
containing exactly one negative edge.
Because we intuitively expect a positive link between two nodes be adjacent to
another positive link (e.g., the transitivity of a friendship relationship
between two individuals),\footnote
{
Observe that, as far as $\Delta(Y)$ is concerned, this need not be true for negative links ---see Proposition~\ref{fa:1}.
This is because the definition of $\Delta(Y)$ does not constrain the number
of clusters of the nodes in $V$. 
}
bad cycles are a clear source of irregularity of the edge labels. The following fact 
relates $\Delta(Y)$ to bad cycles ---see, e.g., \cite{DEFI06}. Figure \ref{f:1} gives a pictorial illustration.
\begin{proposition}\label{fa:1}
For all $(G,Y)$, $\Delta(Y)=0$ iff there are no bad cycles. Moreover, $\Delta(Y)$ is the smallest number of edges that must be removed from $G$ in order to delete all bad cycles.
\end{proposition}
%
Since the removal of an edge can delete more than one bad cycle, $\Delta(Y)$ is upper bounded
by the number of bad cycles in $(G,Y)$. For similar reasons, $\Delta(Y)$ is also lower bounded 
by the number of {\em edge-disjoint} bad cycles in $(G,Y)$.
%
%
We now show\footnote
{
Due to space limitations, all proofs are given in the appendix.
} 
that $Y$ may be very irregular on dense graphs, where $\Delta(Y)$ may take values as big as $\Theta(|E|)$. 
\begin{lemma}\label{fact1}
Given a clique $G = (V,E)$ and any integer $0 \le K \le \tfrac{1}{6}(|V|-3)(|V|-4)$, there exists an edge labeling 
$Y$ such that $\Delta(Y) = K$.
\end{lemma}
%

The restriction of the correlation clustering index to two clusters only leads
to measure the regularity of an edge labeling $Y$ through $\Delta_2(Y)$,
i.e., the minimum cost over all {\em two-cluster} partitions of $V$.
Clearly, $\Delta_2(Y) \ge \Delta(Y)$ for all $Y$.
The fact that, at least in social networks, $\Delta_2(Y)$ tends to be small is motivated by the Cartwright-Harary theory of structural balance (``the enemy of my enemy is my friend'').\footnote{
Other approaches consider different types of local structures, like the contradictory triangles of \cite{LHK10}, or the longer cycles used in \cite{cntd11}.
}
On signed networks, this corresponds to the following {\em multiplicative rule}: 
$Y_{i,j}$ is equal to the product of signs on the edges of \textsl{any} path connecting $i$ to $j$. It is easy 
to verify that, if the multiplicative rule holds for all paths, then $\Delta_2(Y) = 0$. 
It is well known that $\Delta_2$ is related to the \textsl{signed Laplacian} matrix $\sL$ of $(G,Y)$.
Similar to the standard graph Laplacian, the signed Laplacian is defined as 
$\sL = D - Y$, where $D = \diag(d_1,\dots,d_n)$ is the diagonal matrix of node degrees.
Specifically, we have
\begin{equation}\label{eq:mincut}
\frac{4\Delta_2(Y)}{n} = \min_{\bx\in\spin^n} \frac{\bx^{\top}\sL\bx}{\norm{\bx}^2}~.
\end{equation}
Moreover,  $\Delta_2(Y) = 0$ is equivalent to $|\sL| = 0$ ---see, e.g.,~\cite{Hou05}. 
Now, computing $\Delta_2(Y)$ is still NP-hard~\cite{GG06}. Yet, 
%
%
because~(\ref{eq:mincut}) resembles an eigenvalue/eigenvector computation, 
\cite{KLB09} and other authors 
have looked at relaxations similar to those used in spectral graph clustering~\cite{von07}. 
If $\lambdamin$ denotes the smallest eigenvalue of $\sL$, then~(\ref{eq:mincut}) allows one to write
\begin{equation}\label{eq:mincut-relaxed}
\lambdamin =   \min_{\bx\in\R^n} \frac{\bx^{\top}\sL\bx}{\norm{\bx}^2} \le \frac{4\Delta_2(Y)}{n}~.
\end{equation}
%
As in practice one expects $\Delta_2(Y)$ to be strictly positive, 
solving the minimization problem in~(\ref{eq:mincut-relaxed}) amounts to finding an eigenvector 
$\bv$ associated with the smallest eigenvalue of $\sL$. 
The {\em least eigenvalue heuristic} builds $\sL$ out of the training edges only, computes
the associated minimal eigenvector $\bv$,
uses the sign of $\bv$'s components to define a two-clustering of the nodes, 
and then follows this two-clustering to classify all the remaining edges:
Edges connecting nodes with matching signs are classified $+1$, otherwise they
are $-1$. In this sense, this heuristic resembles the transductive risk minimization procedure
described in Subsection~\ref{ss:rade}.
However, no theoretical guarantees are known for such spectral heuristics.

When $\Delta_2$ is the measure of choice, a bad cycle is any simple cycle
containing an {\em odd number} of negative edges. Properties similar to those stated in Lemma~\ref{fact1} can be proven for this new notion of bad cycle.

\section{Mistake bounds and risk analysis}
\label{s:halving}
In this section, we study the prediction complexity of classifying the links of a signed network
in the online and batch transductive settings. Our bounds are expressed in terms of the correlation
clustering index $\Delta$. The $\Delta_2$ index will be used in Section~\ref{s:active} in the context 
of active learning.

\subsection{Online transductive learning}
We first show that, disregarding computational aspects, 
$\Delta(Y)+|V|$ characterizes (up to log factors)
the optimal number of edge classification mistakes in the online trandsuctive learning protocol. 
In this protocol, the edges of the graph are presented to the learner according to an arbitrary and unknown order 
$e_1,\dots,e_T$, where $T = |E|$. At each time $t=1,\dots,T$ the learner receives
edge $e_t$ and must predict its label $Y_t$. Then $Y_t$ is revealed and the learner knows whether 
a mistake occurred. The learner's performance is measured by the total number of prediction mistakes 
on the worst-case order of edges. 
Similar to standard approaches to node classification in networked data \cite{HP07,CGV09a,CGVZ10a,CGVZ10b,VCGZ12}, 
we work within a transductive learning setting. This means that the learner has preliminary access to the entire graph structure $G$ where labels in $Y$ are absent.
We start by showing a lower bound that holds for \textsl{any} online edge classifier, and
then we prove how the lower bound can be strengthened if the graph is dense.
\begin{theorem}\label{t:lower}
For any $G = (V,E)$, any $K \ge 0$, and any online edge classifier, there exists an edge labeling $Y$
on which the classifier makes at least
$
    |V| - 1 + K
$
mistakes, while $\Delta(Y) \leq K$.
\end{theorem}
%
%
%
%
\begin{theorem}\label{t:clique}
For any clique graph $G = (V,E)$, any  $K \ge 0$, and any online classifier, there exists an edge labeling $Y$
on which the classifier makes at least
$
|V| + \max\Bigl\{ K,\ K\log_2\tfrac{|E|}{2K}\Bigr\} + \Omega(1)
$
mistakes, while $\Delta(Y) \leq K$.
\end{theorem}
The above lower bounds are nearly matched by a standard version space algorithm: the Halving algorithm ---see, e.g., 
\cite{Lit89a}.
When applied to link classification, the Halving algorithm with parameter $d$, denoted by $\textsc{hal}_d$, 
predicts the label of edge $e_t$ 
as follows: Let $S_t$ be the number of labelings $Y$ consistent with the observed edges and such that 
$\Delta(Y)=d$ (the version space). 
$\textsc{hal}_d$ predicts $+1$ if the majority of these labelings assigns $+1$ to $e_t$. 
The $+1$ value is also predicted as a default value if either $S_t$ is empty or there is a tie. 
%
%
Otherwise the algorithm predicts $-1$.
Now consider the instance of Halving run with parameter $d^* = \Delta(Y)$ for the true unknown labeling $Y$. 
Since the size of the version space halves after each mistake, this algorithm makes at most $\log_2|S^*|$ 
mistakes, where $S^* = S_1$ is the initial version space of the algorithm.
If the online classifier runs the Weighted Majority algorithm of~\cite{LW94} 
over the set of at most $|E|$
experts corresponding to instances of $\textsc{hal}_d$ for all possible values $d$
of $\Delta(Y)$ (recall Lemma~\ref{fact1}), we easily obtain the following.
\begin{theorem}\label{t:wm}
Consider the Weighted Majority algorithm using $\textsc{hal}_1,\dots,\textsc{hal}_{|E|}$ as experts. The number of mistakes made by this algorithm when run over an arbitrary permutation of edges of a given signed graph $(G,Y)$ is at most of the order of
$
    \bigl(|V|+\Delta(Y)\bigr)\log_2\tfrac{|E|}{\Delta(Y)}
$.
\end{theorem}
Comparing Theorem~\ref{t:wm} to Theorem~\ref{t:lower} and Theorem~\ref{t:clique}
provides our characterization of the prediction complexity of link classification
in the online transductive learning setting.

\paragraph{Computational complexity.}
Unfortunately, as stated in the next theorem, the Halving algorithm for link classification 
is only of theoretical relevance, due to its computational hardness. 
In fact, this is hardly surprising, since $\Delta(Y)$ itself is NP-hard to compute~\cite{BBC04}.
%
%
\begin{theorem}\label{t:hardness}
The Halving algorithm cannot be implemented in polytime unless $\mathrm{RP} = \mathrm{NP}$. 
\end{theorem}
%

\subsection{Batch transductive learning}
\label{ss:rade}
We now prove that $\Delta$ can also 
be used to control the number of prediction mistakes in the batch transductive setting. 
In this setting, given a graph $G = (V,E)$ with unknown labeling $Y\in\spin^{|E|}$ and 
correlation clustering index $\Delta = \Delta(Y)$, the learner observes the labels of a 
random subset of $m$ training edges, and must predict the labels of the remaining $u$ test edges, 
where $m+u=|E|$.

Let $1,\dots,m+u$ be an arbitrary indexing of the edges in $E$. We represent the random set 
of $m$ training edges by the first $m$ elements $Z_1,\dots,Z_m$ in a random permutation 
$Z = (Z_1,\dots,Z_{m+u})$ of $\{1,\dots,m+u\}$. 
Let $\scP(V)$ be the class of all partitions of $V$ and $f\in\scP$ denote a specific (but arbitrary) 
partition of $V$. Partition $f$ predicts the sign of an edge $t\in\{1,\dots,m+u\}$ using 
$\hf(t)\in\spin$, where $\hf(t) = 1$ if $f$ puts the vertices incident to the $t$-th edge of $G$ in 
the same cluster, and $-1$ otherwise.
For the given permutation $Z$, we let $\Delta_m(f)$ denote the cost of the partition 
$f$ on the first $m$ training edges of $Z$ with respect to the underlying edge labeling $Y$. 
In symbols,
$
	\Delta_m(f) = \sum_{t=1}^m \bigl\{\hf(Z_t) \neq Y_{Z_t}\bigr\}.
$
%
%
Similarly, we define $\Delta_u(f)$ as the cost of $f$ on the last $u$ test edges of $Z$,
$
	\Delta_u(f) = \sum_{t=m+1}^{m+u} \bigl\{\hf(Z_t) \neq Y_{Z_t}\bigr\}~.
$
We consider algorithms that, given a permutation $Z$ of the edges, find a partition $f\in\scP$ 
approximately minimizing $\Delta_m(f)$. For those algorithms,  
we are interested in bounding the number $\Delta_u(f)$ of mistakes made when using 
$f$ to predict the test edges $Z_{m+1},\dots,Z_{m+u}$. 
In particular, as for more standard empirical risk minimization schemes, 
we show a bound on the number of mistakes made when predicting the test 
edges using a partition that approximately minimizes $\Delta$ on the training edges.

The result that follows is a direct consequence of \cite{EP09}, and holds for 
any partition that approximately minimizes the correlation clustering index on the training set.
\begin{theorem}\label{th:rade}
Let $(G,Y)$ be a signed graph with $\Delta(Y) = \Delta$.  
Fix $\delta\in (0,1)$, and let $\hf^*\in\scP(V)$ be such that
$
	\Delta_m(\hf^*) \le \kappa\,\min_{f\in\scP}\Delta_m(f)
$
for some $\kappa \ge 1$.
If the permutation $Z$ is drawn uniformly at random, then there exist constants $c, c' > 0$ such that
\begin{center}
$
\frac{1}{u}\Delta_u(\hf^*) 
\le 
\frac{\kappa}{m+u}\Delta + c\sqrt{\left(\frac{1}{m} 
+ \frac{1}{u}\right)\left(|V|\ln|V|+\ln\frac{2}{\delta}\right)}
+ c'\kappa\sqrt{\frac{u/m}{m+u}\ln\frac{2}{\delta}}
$
\end{center}
holds with probability at least $1-\delta$.
\end{theorem}
We can give a more concrete instance of Theorem~\ref{th:rade} 
by using the polynomial-time algorithm of~\cite{DEFI06} which finds a partition $\hf^*\in\scP$ such that
\(
    \Delta_m(\hf^*) \le 3\ln(|V|+1)\min_{f\in\scP}\Delta_m(f)~.
\)
Assuming for simplicity $u = m = \tfrac{1}{2}|E|$, the bound of Theorem~\ref{th:rade} can be rewritten as
\begin{center}
$
\Delta_u(\hf^*) 
\le 
\frac{3}{2}\ln(|V|+1)\Delta + 
\scO\left( \sqrt{|E|\,\left(|V|\,\ln|V| + \ln\frac{1}{\delta}\right)} 
+ \ln|V|\,\sqrt{|E|\,\ln\frac{1}{\delta}}\right)~. 
$
\end{center}
This shows that, when training and test set sizes are comparable, 
approximating $\Delta$ on the training set to within a factor $\ln|V|$ yields 
at most order of $\Delta\ln|V| + \sqrt{|E|\,|V|\,\ln|V|}$ errors on the test set.
Note that for moderate values of $\Delta$ the uniform convergence term $\sqrt{|E|\,|V|\,\ln|V|}$ becomes dominant in the bound.\footnote
{
A very similar analysis can be carried out using $\Delta_2$ instead of $\Delta$. In this case the uniform convergence term
is of the form $\sqrt{|E|\,|V|}$.
}
Although in principle any approximation algorithm with a nontrivial performance guarantee 
can be used to bound the risk, we are not aware of algorithms that are reasonably easy to 
implement and, more importantly, scale to large networks of practical interest.

\section{Two-clustering and active learning}
\label{s:active}
In this section, we exploit the $\Delta_2$ inductive bias to design and analyze 
algorithms in the active learning setting.
%
Active learning algorithms work in two phases: a \textsl{selection} phase, where a query 
set of given size is constructed, and a \textsl{prediction} phase, where the algorithm 
receives the labels of the edges in the query set and predicts the labels of the remaining edges. 
In the protocol we consider here the only labels ever revealed to the algorithm are those in the query set. 
In particular, no labels are revealed during the prediction phase. We evaluate our 
active learning algorithms just by the number of mistakes made in the prediction phase as a function
of the query set size.

Similar to previous sections, 
we first show that the prediction complexity of active learning is lower bounded by the 
correlation clustering index, where we now use $\Delta_2$ instead of $\Delta$. 
In particular, any active learning algorithm for link classification that queries 
at most a constant fraction of the edges must err, on any signed graph, on at 
least order of $\Delta_2(Y)$ test edges, for some labeling $Y$.
\begin{theorem}\label{th:lower-active}
For any signed graph $(G,Y)$, any $K \ge 0$, and any active learning algorithm $A$ for 
link classification that queries the labels of a fraction $\alpha \ge 0$ of the edges of $G$, 
there exists a randomized labeling such that the number $M$ of mistakes made by $A$ in 
the prediction phase satisfies
$
	\E\,M \ge \tfrac{1-\alpha}{2}K,
$
while $\Delta_2(Y) \le K$. 
\end{theorem}
Comparing this bound to that in Theorem~\ref{t:lower} reveals that the active learning 
lower bound seems to drop significantly. Indeed, because the two learning protocols are incomparable
(one is passive online, the other is active batch) so are the two bounds. Besides, 
Theorem~\ref{t:lower} depends
on $\Delta(Y)$ while Theorem~\ref{th:lower-active} depends on the larger quantity $\Delta_2(Y)$.
Next, we design and analyze two efficient active learning algorithms working 
under different assumptions on the way edges are labeled.
%
%
Specifically, we consider two models for generating labelings $Y$: $p$-random and adversarial.
In the $p$-\textsl{random} model, an auxiliary labeling $Y'$ is arbitrarily chosen such that $\Delta_2(Y')=0$. Then $Y$ is obtained through a probabilistic perturbation of $Y'$, where
$
\Pr\bigl(Y_e \neq Y'_e \bigr) \le p
$
for each $e\in E$ (note that correlations between flipped labels are allowed) and for some $p \in [0,1)$.
In the \textsl{adversarial} model, $Y$ is completely arbitrary, and corresponds to an arbitrary partition of $V$ 
made up of two clusters.

\subsection{Random labeling}
Let $\Eflip$ denote the subset of edges $e\in E$ such that $Y_e \neq Y'_e$ in the $p$-random model.
The bounds we prove hold in expectation over the perturbation of $Y'$ and depend on $\E |\Eflip|$ rather than $\Delta_2(Y)$.
Clearly, since each label flip can increase $\Delta_2$ by at most one, then $\Delta_2(Y) \le |\Eflip|$.
Moreover, one can show (details are omitted from this version of the paper)
that there exist classes of dense graphs on which $|\Eflip|=\Delta_2$ with high probability.

\sloppypar{
During the selection phase, our algorithm for the $p$-random model queries only the edges of a
spanning tree $T = (V_T,E_T)$ of $G$. In the prediction phase,
the label of any remaining test edge $e' = (i,j)\not\in E_T$ is predicted with the sign of the product over 
all edges along the unique path $\path_T(e')$ between $i$ and $j$ in $T$.
Clearly, if a test edge $e'$ is predicted wrongly, then either 
$e'\in \Eflip$ or $\path_T(e')$ contains at least one edge of $\Eflip$. 
Hence, the number of mistakes $M_T$ made by our active learner 
on the set of test edges $E\setminus E_T$ can be deterministically bounded by
\begin{equation}\label{e:detbound}
    M_T
\le
    |\Eflip| + \sum_{e' \in E\setminus E_T}\sum_{e \in E} \Ind{e \in \path_T(e')} \Ind{e \in \Eflip}
\end{equation}
where $\Ind{\cdot}$ denotes the indicator of the Boolean predicate at argument.
Let $\bigl|\path_T(e')\bigr|$ denote the number of edges in $\path_T(e')$.
A quantity which can be related to $M_T$ is the {\em average stretch} of a spanning tree $T$
which, for our purposes, reduces to\ \ 
$
    \frac{1}{|E|}\left( |V|-1 + \sum_{e' \in E\setminus E_T} \bigl|\path_T(e')\bigr| \right)~.\ 
$
}
A beautiful result of~\cite{EEST10} shows that every connected and unweighted graph has a spanning tree with an average stretch of just $\mathcal{O}\bigl(\log^2|V|\log\log|V|\bigr)$. Moreover, this low-stretch tree can be constructed in time $\scO\bigl(|E|\ln|V|\bigr)$.
If our active learner uses a spanning tree with the same low stretch, then the following result can be easily proven.
\begin{theorem}\label{th:randomadv}
Let $(G,Y)$ be labeled according to the $p$-random model and assume
the active learner queries the edges of a spanning tree $T$
with average stretch $\mathcal{O}\bigl(\log^2|V|\log\log|V|\bigr)$. Then\ 
\(
    \E\,M_T \le p |E| \times \mathcal{O}\bigl(\log^2|V|\log\log|V|\bigr)~.
\)
\end{theorem}

\subsection{Adversarial labeling}
%
The $p$-random model has two important limitations: first, depending on the graph topology, 
the expected size of $\Eflip$ may be significantly larger than $\Delta_2$. Second, the 
tree-based active learning algorithm for this model works with a fixed query budget of $|V|-1$ 
edges (those of a spanning tree). We now introduce a more sophisticated algorithm for the 
adversarial model which addresses both issues: it has a guaranteed mistake bound expressed 
in terms of $\Delta_2$ and works with an arbitrary budget of edges to query.

Given $(G,Y)$, fix an optimal two-clustering of the nodes with cost $\Delta_2=\Delta_2(Y)$. 
Call \textsl{$\delta$-edge} any edge $(i,j) \in E$ whose sign $Y_{i,j}$ disagrees with this optimal 
two-clustering. Namely, $Y_{i,j} = -1$ if $i$ and $j$ belong to the same cluster, or 
$Y_{i,j} = +1$ if $i$ and $j$ belong to different clusters. 
Let $E_\Delta \ss E$ be the subset of $\delta$-edges.

We need the following ancillary definitions and notation.
Given a graph $G = (V_G,E_G)$, and a rooted subtree $T = (V_T,E_T)$ of $G$,
we denote by $T_i$ the subtree of $T$ rooted at node $i \in V_T$.
Moreover, if $T$ is a tree and $T'$ is a subtree of $T$, both being
in turn subtrees of $G$, we let
\(
E_G(T',T)
\)
be the set of all edges of $E_G \setminus E_T$ that link nodes in $V_{T'}$ 
to nodes in $V_{T} \setminus V_{T'}$. Also, for nodes $i, j \in V_G$,
of a signed graph $(G,Y)$, and tree $T$, we denote by $\pi_T(i,j)$ 
the product over all edge signs along the (unique) path 
$\path_T(i,j)$ between $i$ and $j$ in $T$.
Finally, a {\em circuit} $C= (V_C,E_C)$ (with node set 
$V_C \subseteq V_G$ and edge set $E_C \subseteq E_G$) 
of $G$ is a cycle in $G$. We do not insist on the cycle being simple.
%
%
%
%
%
Given any edge $(i,j)$ belonging to at least one circuit $C = (V_C,E_C)$ of $G$, we let $C_{i,j}$ be
the path obtained by removing edge $(i,j)$ from circuit $C$.
If $C$ contains no $\delta$-edges, then it must be the case that $Y_{i,j} = \pi_{C_{i,j}}(i,j)$.
%
 
Our algorithm finds a {\em circuit covering} $\scC(G)$ of the input graph $G$,
in such a way that each circuit $C \in \scC(G)$
contains at least one edge $(i_C,j_C)$ belonging solely to circuit $C$.
This edge is included in the test set, whose size is therefore equal to $|\scC(G)|$.
The query set contains all remaining edges. During the prediction phase,
each test label $Y_{i_C,j_C}$ is simply predicted with $\pi_{C_{i_C,j_C}}(i_C,j_C)$.
See Figure \ref{f:treepartition} (left) for an example.



For each edge $(i,j)$, let $L_{i,j}$ be the the number of circuits of $\scC(G)$ which
$(i,j)$ belongs to. We call $L_{i,j}$ the \textit{load} of $(i,j)$ induced by $\scC(G)$. 
Since we are facing an adversary, and each $\delta$-edge may give rise to a number
of prediction mistakes which is at most equal to its load, 
one would ideally like to construct a circuit covering $\scC(G)$ minimizing 
$\max_{(i,j) \in E} L_{i,j}$, and such that $|\scC(G)|$ is not smaller than the desired 
test set cardinality.

Our algorithm takes in input a {\em test set-to-query set ratio} $\rho$ and
finds a circuit covering $\scC(G)$ such that: 
(i) $|\scC(G)|$ is the size of the test set, and
(ii) $\frac{|\scC(G)|}{Q - |V_G|+1} \ge \rho$, where $Q$ is the size
of the chosen query set, and
(iii) the maximal load $\max_{(i,j) \in E_G} L_{i,j}$ is $\scO( \rho^{3/2} \sqrt{|V_G|})$.

For the sake of presentation, we first describe a simpler version of our main algorithm. 
This simpler version, called \alg\ (Simplified Constrained Circuit Covering Classifier), 
finds a circuit covering $\scC(G)$ such that 
$\max_{(i,j) \in E} L_{i,j} = \scO( \rho \sqrt{|E_G|})$, and will be used as a subroutine
of the main algorithm. 
 
In a preliminary step, \alg\ draws an {\em arbitrary} spanning tree $T$ of $G$
and queries the labels of all edges of $T$.
Then \alg\ partitions tree $T$ into a small number of 
connected components of $T$.
The labels of the edges $(i,j)$ with $i$ and $j$
in the same component are simply predicted by $\pi_T(i,j)$.
This can be seen to be equivalent to create, for each such edge, a circuit made up of
edge $(i,j)$ and $\path_T(i,j)$.
For each component $T'$, the edges in $E_G(T',T)$ are partitioned
into query set and test set satisfying the given test set-to-query set ratio $\rho$,
so as to increase the load of each queried edge in $E_T \setminus E_{T'}$ by only $\scO(\rho)$. 
Specifically, each test edge  $(i,j) \in E_G(T',T)$ lies on a circuit made up of 
edge $(i,j)$ along with a path contained in $T'$, a path contained in $T \setminus T'$, 
and another edge from $E_G(T',T)$. 
A key aspect to this algorithm is the way of partitioning tree $T$ 
so as to guarantee that the load of each queried edge is $\scO(\rho\sqrt{|E_G|})$. 

In turn, \alg\ relies on two subroutines, \tp\ and \ep, which we now describe. 
Let $T$ be the spanning tree of $G$ drawn in \alg's preliminary step, 
and $T'$ be any subtree of $T$. Let $i_r$ be an arbitrary vertex belonging to both
$V_T$ and $V_{T'}$ and view both trees as rooted at $i_r$.
\begin{figure}[t!]
\begin{center}
\figscale{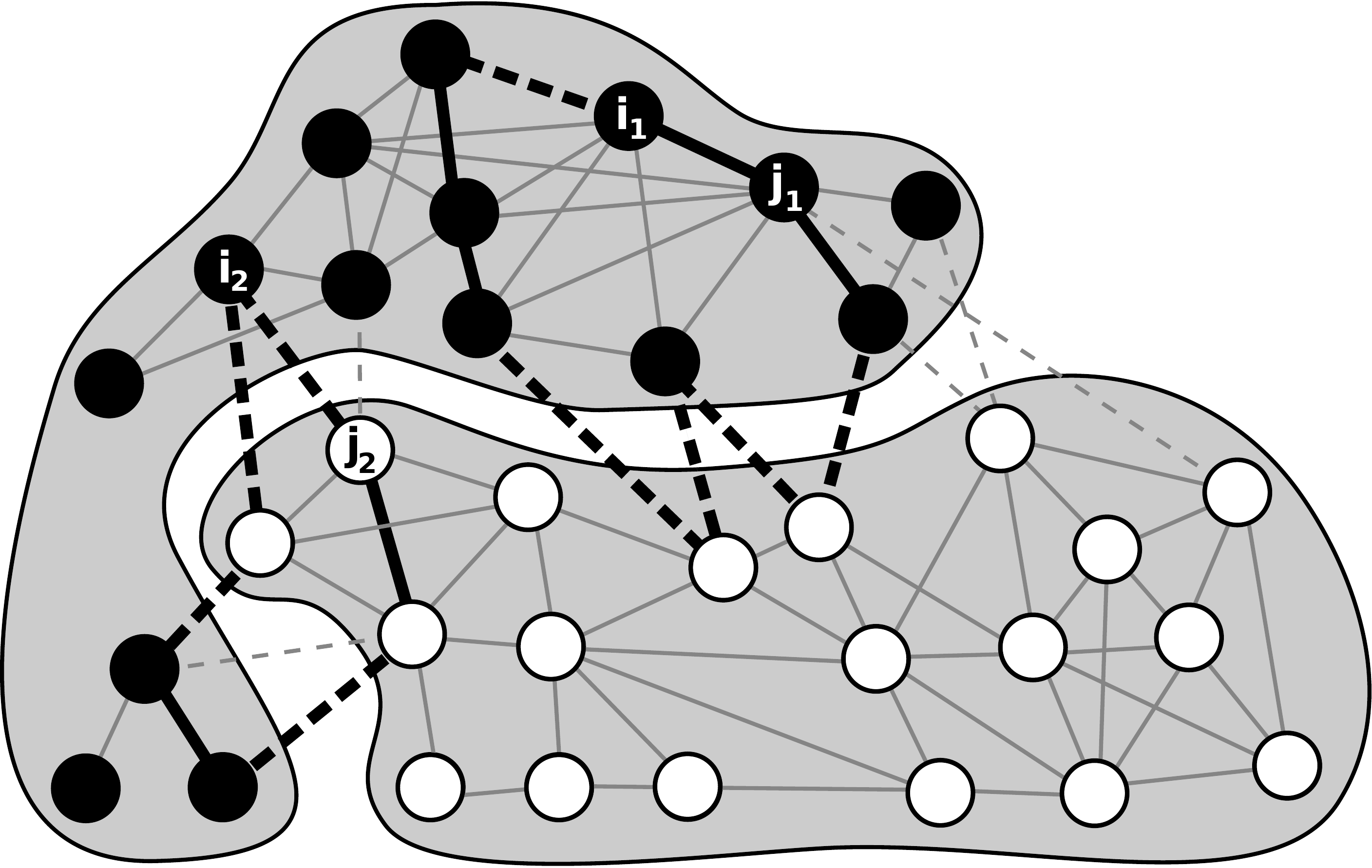}{0.24}\ \ \ 
\figscale{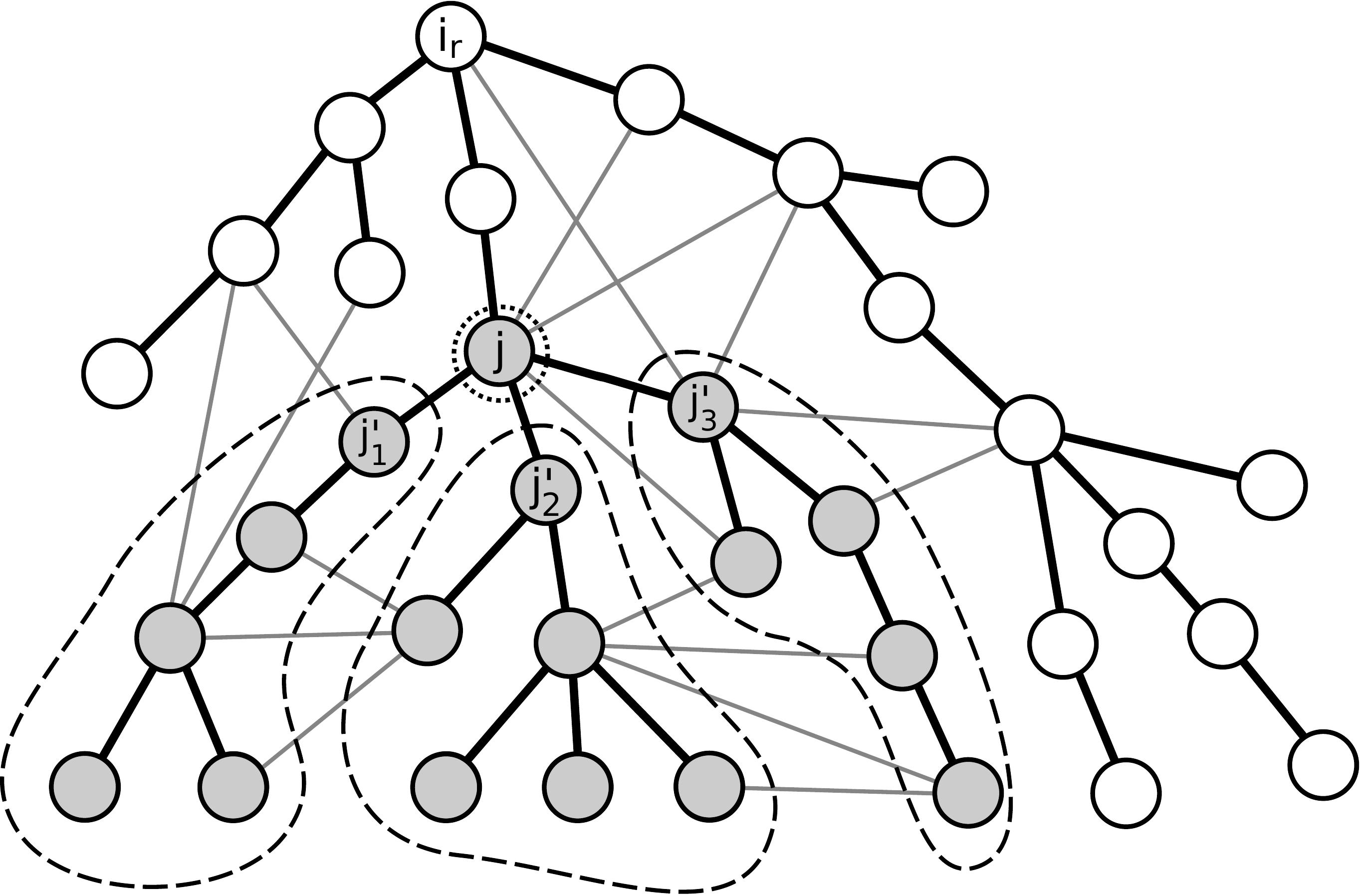}{0.24}
\end{center}
\vspace{-0.1in}
\caption{
\label{f:treepartition}
{\bf Left:} An illustration of how a circuit can be used in the selection and prediction phases. 
An optimal two-cluster partition is shown.
The negative edges are shown using dashed (either thick black or thin gray) lines. 
Two circuits are depicted using thick black lines: one containing edge $(i_1, j_1)$, 
the other containing edge $(i_2, j_2)$. For each circuit $C$ in the graph, 
we can choose any edge belonging to $C$ to be part of the test set, all remaining edge labels being queried.  
In this example, if we select $(i_1,j_1)$ as test set edge, then 
label $Y_{i_1,j_1}$ is predicted with
$(-1)^{5} = -1$, since $5$ is the number of negative edges in $E_C \setminus \{(i_1,j_1)\}$. 
Observe that the presence of a $\delta$-edge (the negative edge incident to $i_1$) 
causes a prediction mistake in this case. 
The edge $Y_{i_2,j_2}$ on the other circuit is predicted correctly, since this 
circuit contains no $\delta$-edges.
{\bf Right:} A graph $G = (V,E)$ and a spanning tree $T = (V_T,E_T)$, rooted at $i_r$, 
whose edges $E_T$ are indicated by thick lines. A subtree $T_j$, rooted at $j$, with grey nodes. 
According to the level order induced by root $i_r$, nodes $j_1'$, $j_2'$ and $j_3'$ are the 
children of node $j$. \tp\ computes $E_G(T_j,T)$, the set of edges connecting one grey node 
to one white node, as explained in the main text.
Since $|E_G(T_j,T)| \ge 9$, \tp\ invoked with $\theta=9$ returns tree $T_j$. 
}
\end{figure}
\tp$(T',G, \theta)$ returns a subtree $T'_j$ of $T'$ such that:
(i) for each node $v \not\equiv j$ of $T'_j$ we have $|E_G(T'_v,T')| \le \theta$, and 
(ii) $|E_G(T_j',T')| \ge \theta$.
In the special case when no such subtree $T'_j$ exists, the whole 
tree $T'$ is returned, i.e., we set $T'_j \equiv T'$. 
As we show in Lemma~\ref{l:tpart} in Appendix \ref{sa:additional}, 
in this special case (i) still holds. 
\tp\ can be described as follows ---see Figure \ref{f:treepartition} (right) for an example. We perform
a depth-first visit of the input tree starting from $i_r$. 
We associate some of the nodes $i \in V_{T'}$ with a
record $R_i$ containing all edges of $E_G(T'_i, T')$. Each time we visit
a leaf node $j$, we insert in $R_j$ all edges linking $j$ to all other nodes in $T'$ 
(except for $j$'s parent). On the other hand, if
$j$ is an \textit{internal} node, when we visit it for the \textsl{last} time,\footnote
{
Since $j$ is an internal node, its last visit is performed during a backtracking step
of the depth-first visit.
}
we set $R_j$ to the union of $R_{j'}$ over all $j$'s children $j'$ ($j'_1$, $j'_2$, and $j'_3$
in Figure \ref{f:treepartition} (right)), 
excluding the edges connecting the subtrees $T_{j'}$ to each other. For instance, in Figure 
\ref{f:treepartition} (right), we include all gray edges departing from subtree $T_j$,
but exclude all those joining the three dashed areas to each other.
In both cases, once $R_j$ is created,
if $|R_j| \ge \theta$ or $j \equiv i_r$, \tp\ stops and
returns $T'_j$. Observe that the order of the depth-first visit ensures that, 
for any internal node $j$, when we are about to compute
$R_j$, all records $R_{j'}$ associated with $j$'s children $j'$ are already available.

We now move on to describe \ep.
Let $i_r \neq q \in V_{T'}$. \ep$(T'_q,T',G,\rho)$ returns a partition 
$\scE=\{E_1, E_2, \ldots\}$ of $E_G(T'_q,T')$ into edge subsets of cardinality $\rho+1$, 
that we call \textit{sheaves}. If $|E_G(T'_q,T')|$ is not a multiple of $\rho+1$, 
the last sheaf can be as large as $2\rho$.
%
%
After this subroutine is invoked, for each sheaf $E_k \in \scE$, \alg\ queries the label of 
an \textit{arbitrary} edge $(i,j) \in E_k$, where
$i \in V_{T'_q}$ and $j \in V_{T'} \setminus V_{T'_q}$. 
Each further edge $(i',j') \in E_k \setminus \{(i,j)\}$, 
having $i' \in V_{T'_q}$ and $j' \in V_{T'}$, will be part of the test set, 
and its label will be predicted using the path\footnote
{
Observe that $\path_{T'_q}(j,j') \equiv  \path_{T'}(j,j') \equiv  \path_{T}(j,j')$, 
since $T'_q \subseteq T' \subseteq T$.
}
$\path_T(i',i) \rightarrow (i,j) \rightarrow \path_T(j,j')$
which, together with test edge $(i',j')$, forms a circuit. Note
that all edges along this path are queried edges and, moreover, $L_{i,j} \leq 2\rho)$ 
because $(i,j)$ cannot belong to more than $2\rho$ circuits of $\scC(G)$. 
Label $Y_{i',j'}$ is therefore predicted by $\pi_T(i',i)\cdot Y_{i,j}\cdot\pi_T(j,j')$.
 
From the above, we see that $E_G(T'_q,T')$ is partioned into test set and query set
with a ratio at least $\rho$.
The partition of $E(T'_q,T')$ into sheaves is carefully performed by \ep\ so as to ensure 
that the load increse of the edges in $E_{T'} \setminus E_{T'_q}$ is only $\scO(\rho)$,
\textit{independent} of the size of $E(T'_q,T')$. 
Moreover, as we show below, the circuits of $\scC(G)$ that we create by invoking 
\ep$(T'_q,T',G,\rho)$ increase the load of each edge of $(u,v)$ (where $u$ is parent of $v$
in $T'_q$) by at most $2\rho\,|E_G(T'_u,T')|$. 
This immediately implies ---see Lemma \ref{l:tpart}(i) in Appendix \ref{sa:additional}--- 
that if $T'_q$ was previously obtained 
by calling \tp$(T', G, \theta)$, then the load of each edge of $T'_q$ gets increased
by at most $2\rho\theta$.
\begin{figure}[t!]
\hrule\vspace{.03in}
\begin{tabbing}
\hspace{.25in} \=\hspace{.10in} \= \hspace{.10in} \=  \hspace{.10in} \=  \hspace{.10in} \=  \hspace{.10in}\= \hspace{.25in} \=\hspace{.10in} \= \hspace{.10in} \=  \hspace{.10in} \=  \kill
\alg$(\rho,\theta)$ \qquad  Parameters: $\rho>0, \theta\ge1$.\\

{\tt 1.}  Draw an arbitrary spanning tree $T$ of $G$, and query all its edge labels\\
{\tt 2.}  \textbf{Do} \\
{\tt 3.} \qquad $T_q \leftarrow \tp(T,G,\theta)$\\
{\tt 4.}  \qquad \textbf{For each} $i,j \in V_{T_q}$, set \ $\hat{Y}_{i,j} \leftarrow \pi_T(i,j)$\\
{\tt 5.}  \qquad $\scE \leftarrow \ep(T_q, T, G, \rho)$\\
{\tt 6.} \qquad \textbf{For each} $E_k \in \scE$\\
{\tt 7.} \qquad \qquad query the label of an arbitrary edge $(i,j) \in E_k$\\
{\tt 8.} \qquad \qquad \textbf{For each} edge $(i',j') \in E_k \setminus \{(i,j)\}$,   
         where $i, i' \in V_{T_q}$ and $j,j' \in V_{T} \setminus V_{T_q}$\\
{\tt 9.}  \qquad \qquad \qquad    $\hat{Y}_{i',j'} \leftarrow \pi_T(i',i)\cdot Y_{i,j}\cdot \pi_T(j,j')$\\
{\tt 10.} \qquad $T \leftarrow T \setminus T_q$\\
{\tt 11.} \textbf{While} ($V_{T} \not\equiv \emptyset$)\\
\end{tabbing}
\vspace{-.3in}\hrule 
\caption{\label{f:algsimple}The Simplified Constrained Circuit Covering Classifier \alg.} 
\end{figure}
We now describe how \ep\ builds the partition of $E_G(T'_q,T')$ into sheaves. 
\ep$(T'_q,T',G,\rho)$ first performs a depth-first visit of 
$T' \setminus T'_q$ starting from root $i_r$. 
Then the edges of $E_G(T'_q,T')$ are numbered consecutively by the order of this visit, 
where the relative ordering of the edges incident to the same node encountered during this 
visit can be set arbitrarily.
Figure \ref{f:edgepartition} in Appendix \ref{sa:additional} helps visualizing the process of sheaf construction.
One edge per sheaf is queried, the remaining ones are assigned to the test set. 

\alg's pseudocode is given in Figure~\ref{f:algsimple}, $\hat{Y}_{i,j}$ therein
denoting the predicted labels of the test set edges $(i,j)$.
The algorithm takes in input the ratio parameter $\rho$ (ruling the test set-to-query set ratio), 
and the threshold parameter $\theta$. 
%
%
After drawing an initial spanning tree $T$ of $G$, and querying all its edge labels, 
\alg\ proceeds in steps as follows.
At each step, the algorithm calls \tp\ on (the current) $T$. 
Then the labels of all edges linking pairs of nodes 
$i,j \in V_{T_q}$ are selected to be part of the test set,
and are simply predicted by $\hat{Y}_{i,j} \leftarrow \pi_T(i,j)$. Then,
all edges linking the nodes in $V_{T_q}$ to the nodes in $V_T \setminus V_{T_q}$
are split into sheaves via \ep. For each sheaf, an arbitrary edge $(i,j)$
is selected to be part of the query set. All remaining edges $(i',j')$ become
part of the test set, and their labels are predicted by 
$\hat{Y}_{i',j'} \leftarrow \pi_T(i',i)\cdot Y_{i,j}\cdot \pi_T(j,j')$, 
where $i,i' \in V_{T_q}$
and $j, j' \in V_T \setminus V_{T_q}$. Finally, we shrink $T$ as $T \setminus T_q$, and
iterate until $V_T \equiv \emptyset$. 
Observe that in the last
do-while loop execution we have $T \equiv \tp(T,G,\theta)$. 
Moreover, if $T_q \equiv T$, Lines 6--9 are not executed since 
$E_G(T_q,T) \equiv \emptyset$, which implies that $\scE$ is an empty set partition.
\begin{figure}[!t]
\hrule\vspace{.03in}
\begin{tabbing}
\hspace{.25in} \=\hspace{.10in} \= \hspace{.10in} \=  \hspace{.10in} \=  \hspace{.10in} \=  \hspace{.10in}\= \hspace{.25in} \=\hspace{.10in} \= \hspace{.10in} \=  \hspace{.10in} \=  \kill
\algb$(\rho)$ \qquad  Parameter: $\rho$ satisfying $3 < \rho \leq \frac{|E_G|}{|V_G|}$.\\
{\tt 1.} Initialize $E \leftarrow E_G$\\
{\tt 2.}  \textbf{Do} \\
{\tt 3.}  \qquad Select an arbitrary edge subset $E' \subseteq E$ such that $|E'| = \min\{|E|,\rho|V_G|\}$\\
{\tt 4.}  \qquad Let $G' = (V_G,E')$\\
{\tt 5.}  \qquad \textbf{For each} connected component $G''$ of $G'$, \ run  $\alg(\rho,\sqrt{|E'|})$ on $G''$\\
{\tt 6.}  \qquad $E \leftarrow E \setminus E'$\\
{\tt 7.}  \textbf{While} ($E \not\equiv \emptyset$)
\end{tabbing}
\vspace{-.13in}\hrule 
\caption{\label{f:algfinal}
The Constrained Circuit Covering Classifier \algb.} 
\end{figure}

We are now in a position to describe a more refined algorithm, called \algb\ (Constrained Circuit Covering
Classifier --- see Figure \ref{f:algfinal}),
that uses \alg\ on suitably chosen subgraphs of the original graph. The advantage of \algb\ over
\alg\ is that we are afforded to reduce the mistake bound from 
$\scO\bigl(\Delta_2(Y)\,\rho\,\sqrt{|E_G|}\bigr)$ (Lemma \ref{t:alg} in Appendix \ref{sa:additional})
to $\scO\bigl(\Delta_2(Y)\,\rho^{\frac{3}{2}}\, \sqrt{|V_G|}\bigr)$.
\algb\  proceeds in (at most) $|E_G|/(\rho|V_G|)$ steps as follows. At each step
the algorithm splits into query set and test set an edge subset $E' \subseteq E$,
where $E$ is initially $E_G$. The size of $E'$ is guaranteed to be at most $\rho|V_G|$. 
The edge subset $E'$ is made up of arbitrarily chosen edges 
that have not been split yet into query and test set. 
The algorithm considers subgraph $G' = (V_G,E')$, and invokes
\alg\ on it for querying and predicting its edges.
Since $G'$ can be disconnected, \algb\ simply invokes $\alg(\rho,\sqrt{|E'|})$ 
on each connected component of $G'$.
The labels of the test edges in $E'$
are then predicted, and $E$ is shrunk to $E \setminus E'$.
The algorithm terminates when $E \equiv \emptyset$.
\begin{theorem}\label{t:algb}
The number of mistakes made by $\algb(\rho)$, with $\rho$ satisfying $3 < \rho \leq \frac{|E_G|}{|V_G|}$,
on a graph $G = (V_G,E_G)$ with unknown labeling $Y$ 
is $\scO\bigl(\Delta_2(Y) \rho^{\frac{3}{2}} \sqrt{|V_G|}\bigr)$.
Moreover, we have
$\frac{|\scC(G)|}{Q} \ge \frac{\rho-3}{3}$, where $Q$ is the size
of the query set
and $|\scC(G)|$ is the size of the test set.
\end{theorem}
\newcommand{\scD}{\mathcal{D}}
\begin{remark}
Since we are facing a worst-case (but oblivious) adversary,
one may wonder whether randomization might be beneficial in \alg\ or \algb.
We answer in the affermative as follows.
The \textit{randomized} version of $\alg$ is $\alg$ where the following two steps are randomized:
(i) The initial spanning tree $T$ (Line 1 in Figure \ref{f:algsimple}) is drawn at {\em random} 
according to a given distribution $\scD$ over the spanning trees of $G$. 
(ii) The queried edge selected from each sheaf $E_k$ returned by calling \ep\ 
(Line $7$ in Figure \ref{f:algsimple}) 
is chosen uniformly at random among all edges in $E_k$.
Because the adversarial labeling is oblivious to the query set selection, 
the mistake bound of this randomized $\alg$ can be shown 
to be the sum of the \textit{expected} loads of each $\delta$-edge, 
which can be bounded by \
\(
\scO\Bigl(\Delta_2(Y) \max\{1,\rho\,P_{\scD}^{\max}\,\sqrt{|E_G|-|V_G|+1} \}\Bigr)
\),\ 
where $P_{\scD}^{\max}$ is the maximal over all probabilities of including edges $(i,j) \in E$ in $T$. 
%
When $T$ is a uniformly generated random spanning tree \cite{LP09}, and $\rho$ is a constant
(i.e., the test set is a constant fraction of the query set)
this implies optimality up to a factor $k\rho$ (compare to Theorem \ref{th:lower-active}) on
any graph where the effective resistance \cite{LP09} between any pair of adjacent nodes 
in $G$ is $\scO\bigl(k/|V|\bigr)$ ---for instance, a very dense clique-like graph.   
One could also extend this result to $\algb$, but
this makes it harder to select the parameters of \alg\ within \algb.
\end{remark}
We conclude with some remarks on the time/space requirements for the two algorithms 
\alg\ and \algb, details 
will be given in the full version of this paper. 
The amortized time per prediction required by $\alg(\rho,\sqrt{|E_G|-|V_G|+1})$ 
and $\algb(\rho)$ is
$\scO\Bigl(\frac{|V_G|}{\sqrt{|E_G|}}\log |V_G|\Bigr)$ and 
$\scO\Bigl(\sqrt{\frac{|V_G|}{\rho}}\log |V_G|\Bigr)$, respectively, 
provided $|\scC(G)| = \Omega(|E_G|)$ and $\rho \le \sqrt{|V_G|}$. For instance, 
when the input graph $G = (V_G,E_G)$ 
has a quadratic number of edges, $\alg$ has  
an amortized time per prediction which is only {\em logarithmic} in $|V_G|$. 
In all cases, both algorithms need linear space in the size of the input graph.
In addition, each do-while loop execution within \algb\ can be run in parallel.

\section{Conclusions and ongoing research}
In this paper we initiated a rigorous study of link classification in signed graphs. Motivated by social balance theory, we adopted the correlation clustering index as a natural regularity measure for the problem. We proved upper and lower bounds on the number of prediction mistakes in three fundamental transductive learning models: online, batch and active. Our main algorithmic contribution is for the active model, where we introduced a new family of algorithms based on the notion of circuit covering. Our algorithms are efficient, relatively easy to implement, and have mistake bounds that hold on any signed graph. We are currently working on extensions of our techniques based on recursive decompositions of the input graph. Experiments on social network datasets are also in progress.

\appendix

\section{Proofs}

\begin{proof}[Lemma \ref{fact1}]
The edge set of a clique $G = (V,E)$ can be decomposed into edge-disjoint triangles if and only if there exists 
an integer $k \ge 0$ such that $|V|=6k+1$ or $|V|=6k+3$ ---see, e.g., page~113 of~\cite{Bol86}. 
This implies that for any clique $G$ we can find a subgraph $G' = (V',E')$ such that $G'$ is a 
clique, $|V'| \ge |V|-3$, and $E'$ can be decomposed into edge-disjoint triangles. 
As a consequence, we can find $K$ edge-disjoint triangles among the 
$\tfrac{1}{3}|E'| = \tfrac{1}{6}|V'|(|V'|-1) \ge \tfrac{1}{6}(|V|-3)(|V|-4)$ edge-disjoint triangles of $G'$, 
and label one edge (chosen arbitrarily) of each triangle with $-1$, all the remaining $|E|-K$ edges of $G$
being labeled $+1$. Since the elimination of the $K$ edges labeled $-1$ implies the elimination of all bad cycles, 
we have $\Delta(Y) \leq K$. Finally, since $\Delta(Y)$ is also lower bounded by the number of edge-disjoint 
bad cycles, we also have $\Delta(Y) \geq K$.
\end{proof}

\begin{proof}[Theorem \ref{t:lower}]
The adversary first queries the edges of a spanning tree of $G$ forcing a mistake at each step.
Then there exists a labeling of the remaining edges such that the overall labeling $Y$ satisfies
$\Delta(Y) = 0$.
This is done as follows. We partition the set of nodes $V$ into clusters such that each pair of
nodes in the same cluster is
connected by a path of positive edges on the tree. Then we label $+1$ all non-tree edges that are
incident to nodes in the same cluster, and label $-1$ all non-tree edges that are incident
to nodes in different clusters. Note that in both cases no bad cycles are created, thus
$\Delta(Y) = 0$. After this
first phase, the adversary can force additional $K$ mistakes by querying $K$ arbitrary non-tree
edges and forcing a mistake at each step. Let $Y'$ be the final labeling.
Since we started from $Y$ such that $\Delta(Y)=0$ and at most $K$ edges have been flipped,
it must be the case that $\Delta(Y') \le K$. 
\end{proof}

\begin{proof}[Theorem \ref{t:clique}]
If $K \ge \frac{|V|}{8}$ we have $|V|+K \ge |V| + K \Bigl( \log_2 \frac{|V|}{K} -2 \Bigr)$, so one can prove
the statement just by resorting to the adversarial strategy in the proof of Theorem \ref{t:lower}. 
Hence, we continue by assuming $K < \frac{|V|}{8}$. 
We first show that on a special kind of graph $G' = (V',E')$, whose labels $Y'$ are  
partially revealed, any algorithm can be forced to make at least $\log_2 |V'|$ mistakes with $\Delta(Y') = 1$. 
Then we show (Phase 1) how to force $|V|-1$ mistakes on $G$ while maintaining $\Delta(Y) = 0$, and 
(Phase 2) how to extract from the input graph $G$, consistently with the labels revealed in Phase~1, 
$K$ edge-disjoint copies of $G'$. The creation of each of these subgraphs, 
which contain $|V'| = 2^{\lfloor \log_2 (|V|/(2K))\rfloor}$ nodes, contributes  
$\Bigl\lfloor \log_2 \frac{|V'|}{2K}\Bigr\rfloor \ge  \log_2 \frac{|V|}{K} - 2$ 
additional mistakes. In Phase~2 the value of $\Delta(Y)$ is increased by 
one for each copy of $G'$ extracted from $G$. 

Let $d(i,j)$ be the distance between node $i$ and node $j$ in the graph under consideration, 
i.e., the number of edges in the shortest path connecting $i$ to $j$. 
The graph $G'= (V',E')$ is constructed as follows. The number of nodes $|V'|$ is a power of $2$.
$G'$ contains a cycle graph $C$ having $|V'|$ edges, together with $|V'|/2$ additional edges. 
Each of these additional edges connects the $|V'|/2$ pairs of nodes 
$\{i_0,j_0\}, \{i_1,j_1\}, \dots$ of $V'$ such that, for all indices $k \ge 0$, 
the distance $d(i_k,j_k)$ calculated on $C$ is equal to $|V'|/2$. We say that 
$i_k$ and $j_k$ are \textit{opposite} to one another. 
One edge of $C$, say $(i_0, i_1)$, is labeled $-1$, all the remaining $|V'|-1$ 
in $C$ are labeled $+1$. All other edges of $G'$, which connect opposite nodes, 
are unlabeled. We number the nodes of $V'$ as 
$i_0, i_1, \dots, i_{|V'|/2-1}, j_0, j_1, \ldots, j_{|V'|/2-1}$ in such a way that, on the cycle graph 
$C$, $i_k$ and $j_k$ are adjacent to $i_{k-1}$ and $j_{k-1}$, respectively, for	all indices $k \ge 0$. 
With the labels assigned so far we clearly have $\Delta(Y') = 1$. 

We now show how the adversary can force $\log_2 |V'|$ mistakes upon revealing the unassigned labels, 
without increasing the value of $\Delta(Y')$. The basic idea is to have a version space $S'$ of $G'$,
and halve it at each mistake of the algorithm.
Since each edge of $C$ can be the (unique) $\delta$-edge,\footnote{
Here a $\delta$-edge is a labeled edge contributing to $\Delta(Y)$.
}
we initially have $|S'|=|V'|$.
The adversary forces the first mistake on edge $(i_0,j_0)$, just by assigning a label which is 
different from the one predicted by the algorithm. 
If the assigned label is $+1$ then the $\delta$-edge is constrained to be along 
the path of $C$ connecting $i_0$ to $j_0$ via $i_1$, otherwise it must be along 
the other path of $C$ connecting $i_0$ to $j_0$. 
Let now $L$ be the line graph including all edges that can 
be the $\delta$-edge at this stage, and $u$ be the node in the "middle" of 
$L$, (i.e., $u$ is equidistant from the two terminal nodes). 
The adversary forces a second mistake by asking for the label of the edge connecting $u$ to its opposite
node. If the assigned label is $+1$ then the $\delta$-edge is constrained to be on the half of $L$ which 
is closest to $i_0$, otherwise it must be on the other half. Proceeding this way, the adversary 
forces $\log_2 |V'|$ mistakes without increasing the value of $\Delta(Y')$. 
Finally, the adversary can assign all the remaining $|V'|/2-\log_2|V'|$ labels in such a 
way that the value of $\Delta(Y')$ does not increase. Indeed, after the last forced mistake 
we have $|S'|=1$, and the $\delta$-edge is completely determined. All nodes of the labeled graph obtained 
by flipping the label of the $\delta$-edge can be partitioned into clusters such that each pair of 
nodes in the same cluster is connected by a path of $+1$-labeled edges. Hence the adversary can 
label all edges in the same cluster with $+1$ and all edges 
connecting nodes in different clusters with $-1$. Clearly, dropping the $\delta$-edge resulting 
from the dichotomic procedure also removes all bad cycles from $G'$.

{\noindent \bf Phase 1. }
Let now $H$ be any Hamiltonian path in $G$. In this phase,
the labels of the edges in $H$ are presented to the learner, and one mistake per edge is forced, 
i.e., a total of $|V|-1$ mistakes. According to the assigned labels, the nodes in 
$V$ can be partitioned into two clusters such that any pair of nodes in each cluster are connected by a path 
in $H$ containing an even number of $-1$-labeled edges.\footnote{
In the special case when there is only one cluster, we can think of the second cluster as the empty set.
}  
Let now $V_0$ be the larger cluster and $v_1$ be one 
of the two terminal nodes of $H$. We number the nodes of $V_0$ as $v_1, v_2, \ldots,$ in such a way that 
$v_k$ is the $k$-th node closest to $v_1$ on $H$. 
Clearly, all edges  $(v_k, v_{k+1})$ either have been labeled $+1$ in this phase or are unlabeled. 
For all indices $k \ge 0$, the adversary assigns each unlabeled edge $(v_k, v_{k+1})$ of 
$G$ a $+1$ label. Note that, at this stage, no bad cycles are created, since the edges just labeled 
are connected through a path containing two $-1$ edges. 

{\noindent \bf Phase 2. }
Let $H_0$ be the line graph containing all nodes of $V_0$ and all edges incident to these nodes that have been 
labeled so far. Observe that all edges in $H_0$ are $+1$. Since $|V_0| \ge |V|/2$, $H_0$ must contain a 
set of $K$ edge-disjoint line graphs having $2^{\lfloor \log_2 (|V|/(2\Delta)) \rfloor} $ edges.  
The adversary then assigns label $-1$ to all edges of $G$ connecting the two terminal nodes of 
these sub-line graphs. Consider now all the cycle graphs formed by all node sets of the $K$ sub-line graphs 
created in the last step, together with the $-1$ edges linking the two terminal nodes of each sub-line graph. 
Each of these cycle graphs has a number of nodes which is a power of $2$. Moreover, only one edge is $-1$, 
all remaining ones being $+1$. Since no edge connecting the nodes of the cycles has been assigned yet, 
the adversary can use the same dichotomic technique as above to force, for each cycle graph, 
$\Bigl\lfloor \log_2 \frac{|V|}{2K} \Bigr\rfloor$ additional mistakes without increasing the value of $\Delta(Y)$.
\end{proof}

\begin{proof}[Theorem \ref{t:wm}]
We first claim that the following bound on the version space size holds:
\[
    \log_2|S^*| < d^*\log_2\frac{e|E|}{d^*} + |V|\log_2\frac{|V|}{\ln(|V|+1)}~.
\]
To prove this claim, observe that each element of $S^*$ is uniquely identified by 
a partition of $V$ and a choice of $d^*$ edges in $E$. Let $B_n$ (the Bell number) 
be the number of partitions of a set of $n$ elements. Then
\(
    |S^*| \le B_{|V|} \times \binom{|E|}{d^*}~.
\)
Using the upper bound
$
    B_n < \left(\frac{n}{\ln(n+1)}\right)^n
$
and standard binomial inequalities yields the claimed bound on the version space size. 

Given the above, the mistake bound of the resulting algorithm is an easy consequence 
of the known mistake bounds for Weighted Majority.
\end{proof}

\begin{proof}[Theorem \ref{t:hardness}, sketch]
We start by showing that the Halving algorithm is able to solve UCC by building a reduction from UCC to link
 classification. Given an instance of $(G,Y)$ of UCC, let the supergraph $G' = (V',E')$ be defined as follows: 
introduce $G'' = (V'',E'')$, a copy of $G$, and let $V' = V \cup V''$. Then connect each node $i'' \in V''$ to the 
corresponding node $i \in V$, add the resulting edge $(i,i'')$ to $E'$, and label it with $+1$. Then add to $E'$ all 
edges in $E$, retaining their labels.
Since the optimal clustering of $(G,Y)$ is unique, there exists only one assignment $Y''$ to the labels of $E'' 
\subset E'$ such that $\Delta(Y') = \Delta(Y)$. This is the labeling consistent with the optimal clustering $\scC^*$ 
of $(G,Y)$: each edge $(i'',j'') \in E''$ is labeled $+1$ if the corresponding edge $(i,j) \in E$ connects 
two nodes contained in the same cluster of $\scC^*$, and $-1$ otherwise.
Clearly, if we can classify correctly the edges of $E''$, then the optimal clustering $\scC^*$ is recovered. 
In order  to do so, we run $\textsc{hal}_d$ on $G''$ with increasing values of $d$ starting from $d=0$.
For each value of $d$, we feed {\em all} edges of $G''$ to $\textsc{hal}_d$ and check whether the 
number $z$ of edges $(i'',j'') \in E''$ for which the predicted label is different from the one of the corresponding
edge $(i,j) \in E$, is equal to $d$. 
If it does not, we increase $d$ by one and repeat. The smallest value $d^*$ of $d$ for which $z$ 
is equal to $d$ must be the true value of 
$\Delta(Y')$. Indeed, for all $d < d^*$ Halving cannot find a labeling of $G'$ with cost $d$. Then, we run 
$\textsc{hal}_{d^*}$ on $G'$ and feed each edge of $E''$. After each prediction we reset the algorithm. 
Since the assignment $Y''$ is unique, there is only one labeling (the correct one) in the version space. Hence the 
predictions of $\textsc{hal}_{d^*}$ are all correct, revealing the optimal clustering $\scC^*$.
The proof is concluded by constructing the series of reductions 
\[
\text{Unique Maximum Clique} \rightarrow \text{Vertex Cover} \rightarrow
\text{Multicut} \rightarrow \text{Correlation Clustering}, 
\]
where the initial problem in this chain is known not to 
be solvable in polynomial time, unless $\mathrm{RP} = \mathrm{NP}$.
\end{proof}

\begin{proof}[Theorem \ref{th:rade}]
First, by a straightforward combination of~\cite[Remark~2]{EP09} and the union bound,
we have the following uniform convergence result for the class $\scP$:
With probability at least $1-\delta$, uniformly over $f\in\scP$, it holds that
\begin{equation}\label{e:rade}
\frac{1}{u}\Delta_u(f) 
\le 
\frac{1}{m}\Delta_m(f) 
+ c\sqrt{\left(\frac{1}{m} + \frac{1}{u}\right)\left(|V|\ln|V|+\ln\frac{1}{\delta}\right)}
\end{equation}
where $c$ is a suitable constant.
Then, we let $f^*\in\scP$ be the partition that achieves $\Delta_{m+u}(f^*) = \Delta$.
By applying~\cite[Remark~3]{EP09} we obtain that
\[
	\frac{1}{m}\Delta_m(\hf) \le \frac{\kappa}{m}\Delta_m(f^*) \le \frac{\kappa}{m+u}\Delta_{m+u}(f^*) + c'\kappa\sqrt{\frac{u/m}{m+u}\ln\frac{2}{\delta}}
\]
with probability at least $1 - \tfrac{\delta}{2}$.
An application of (\ref{e:rade}) concludes the proof.
\end{proof}

\begin{proof}[Theorem \ref{th:lower-active}]
Let $Y$ be the following randomized labeling: All edges are labeled $+1$, except 
for a pool of $K$ edges, selected uniformly at random, and whose labels are set randomly. 
Since the size of the 
training set chosen by $A$ is not larger than $\alpha|E|$, the test set will contain in 
expectation at least 
$(1-\alpha)K$ 
randomly labeled edges. Algorithm $A$ makes in expectation $1/2$ mistakes on every such edge. 
Now, if we delete the edges with random labels we obtain a graph with all positive labels,
which immediately implies $\Delta_2(Y) \le K$.
\end{proof}

\begin{proof}[Theorem \ref{th:randomadv}]
We start from (\ref{e:detbound}) and take expectations. We have
\begin{align*}
    \E\,M_T
&\le
    p|E| + \sum_{e' \in E\setminus E_T}\sum_{e \in E} \Ind{e \in \path_T(e')} \Pr\bigl(e \in \Eflip\bigr)\\
&=
    p|E| + p \sum_{e' \in E\setminus E_T} \bigl|\path_T(e')\bigr|
\\ &=
    p|E| + p|E|\times \mathcal{O}\bigl(\log^2|V|\log\log|V|\bigr)\\
&=
    p|E| \times \mathcal{O}\bigl(\log^2|V|\log\log|V|\bigr)
\end{align*}
as claimed.
\end{proof}

\section{Additional figures, lemmas and proofs from Section~\ref{s:active}}\label{sa:additional}

\begin{figure}[t!]
\begin{center}
\includegraphics[width=75mm]{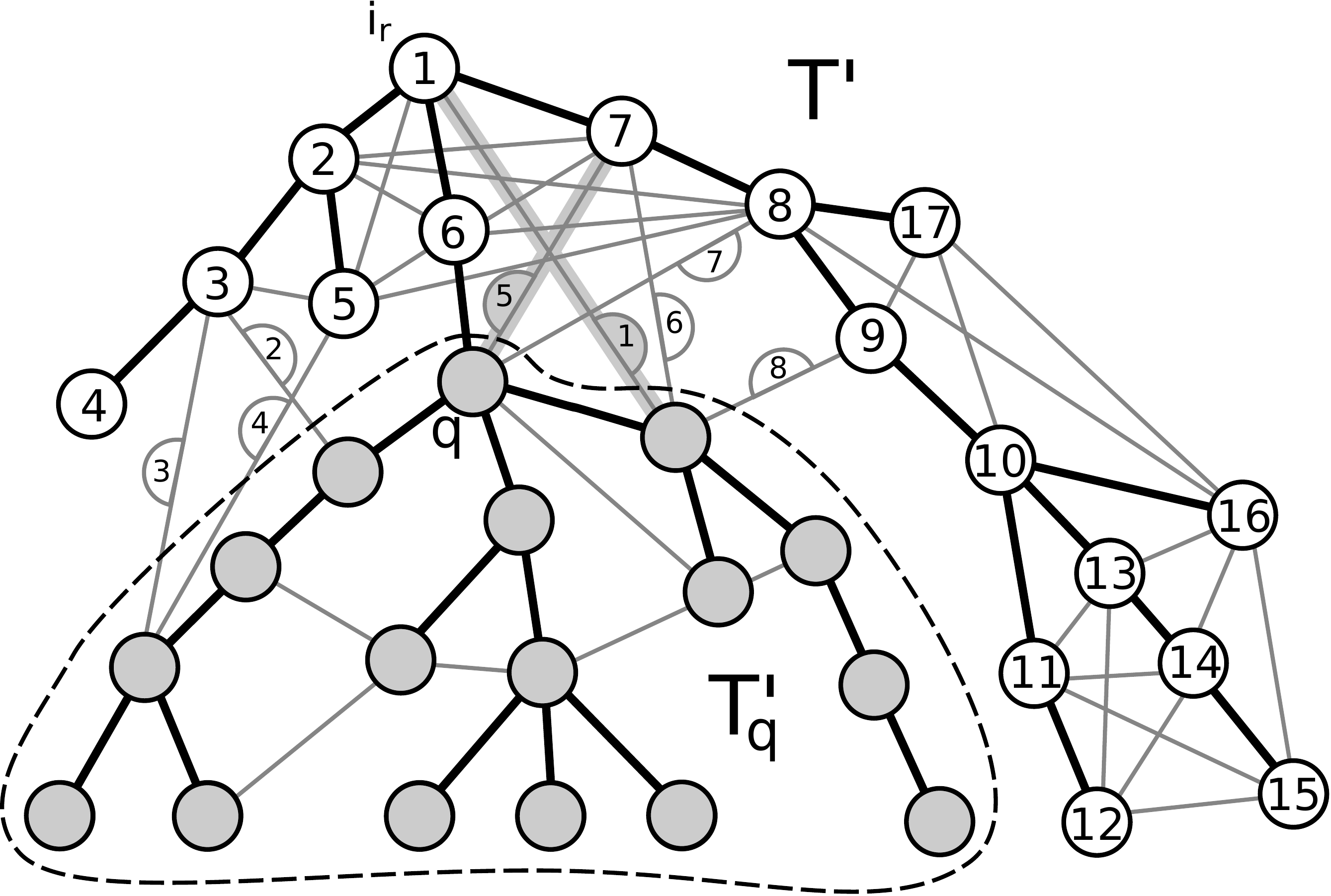}
\end{center}
\caption{
\label{f:edgepartition}
An illustration of how \ep\ operates and how the \alg\ prediction rule uses the output of \ep. 
Notation is as in the main text.
In this example, \ep\ is invoked with parameters $T'_q,T',G, 3$. 
The nodes in $V_{T_q'}$ are grey. The edges of $E_{T'}$ are thick black.
A depth-first visit starting from $i_r$ is performed, and each node is numbered according to 
the usual visit order. 
Let $i_1,i_2, \ldots$ be the nodes in $V_{T'}\setminus V_{T_q'} $ (the white nodes), 
and $j_1,j_2, \ldots$ be the nodes of $V_{T_q'}$ (the gray nodes). 
The depth-first visit on the nodes induces an ordering on the $8$ edges of 
$E_G(T'_q,T')$  (i.e., the edges connecting grey nodes to white nodes), 
where $(i_1,j_1)$ precedes $(i_2,j_2)$ if $i_1$ is visited before $i_2$,
being $i_1,i_2, \ldots$ and $j_1,j_2, \ldots$ belonging to $V_{T' \setminus T'_q}$ and
   $V_{T'_q}$ respectively.
The numbers tagging the edges of $E_G(T'_q,T')$ denote a possible edge ordering.  
In the special case when two or more edges are incident to the same white node, 
the relative order for this edge subset is arbitrary. This way of ordering edges 
is then used by \ep\ for building the partition. Since $\rho = 3$, 
%
%
\ep\ partitions $E_G(T'_q,T')$ in $8/(3+1) = 2$ sheaves, 
the first one containing edges tagged $1, 2, 3, 4$ and the second one with edges
$5, 6, 7, 8$. Finally, from each sheaf an arbitrary edge is queried. A possible selection, shown by grey shaded edges, is edge $1$ for the first sheaf
and edge $5$ for the second one. The remaining edges are assigned to the test set.
For example, during the prediction phase, test edge $(19,26)$ 
is predicted as $\hat{Y}_{19,26} \leftarrow {Y}_{19,7}{Y}_{7,24}{Y}_{24,25}{Y}_{25,26}$,
where ${Y}_{19,7}$, ${Y}_{24,25}$, and ${Y}_{25,26}$ are available since they belong to
the initial spanning tree $T'$.
}
\end{figure}

\begin{lemma}\label{l:tpart}
Let $T'$ be any subtree of $T$, rooted at (an arbitrary) vertex $i_r \in V_{T'}$.
If for any node $v \in V_{T'}$ we have $|E_G(T'_v,T')| \le \theta$, then
\tp$(T',G, \theta)$ returns $T'$ itself; 
otherwise, \tp$(T',G, \theta)$ returns a proper subtree $T'_j \subset T'$ satisfying
\begin{enumerate}
\item [(i)] $|E_G(T'_v,T')| \le \theta$ for each node $v \not\equiv j$ in $V_{T'_j}$, and
\item [(ii)] $|E_G(T_j',T')| \ge \theta$.
\end{enumerate}
\end{lemma}
\begin{proof}
The proof immediately follows from the definition of \tp$(T',G, \theta)$. 
If $|E_G(T'_v,T')| \le \theta$ holds for each node $v \in V_{T'}$, 
then \tp$(T',G, \theta)$ stops only after all nodes of $V_{T'}$ have
been visited, therefore returning the whole input tree $T'$. 
On the other hand, when  $|E_G(T'_v,T')| \le \theta$ does not hold 
for each node $v \in V_{T'}$, \tp$(T',G, \theta)$ returns a proper subtree 
$T'_j \subset T'$ and, by the very way this subroutine works, we must have
$|E_G(T_j',T')| \ge \theta$. This proves (ii). In order to prove (i), assume,
for the sake of contradiction, that there exists
a node $v \not\equiv j$ of $V_{T'_j}$ such that $|E_G(T'_v,T')| > \theta$. Since
$v$ is a descendent of $j$,  the last time when $v$ gets visited precedes the last time when
$j$ does, thereby implying that
\tp\ would stop at some node $z$ of $V_{T'_v}$, which would make \tp$(T',G, \theta)$
return $T'_z$ instead of $T'_j$.
\end{proof}

Let $\scC(T'_q,T',G,\rho) \subseteq \scC(G)$ be the set of circuits used during 
the prediction phase that have been obtained thorugh the sheaves 
$\{E_1, E_2, \ldots\}$ returned by $\ep(T'_q,T',G,\rho)$. The following lemma quantifies the resulting load increase of the edges in $E_{T'} \setminus E_{T'_q}$.
\begin{lemma}\label{l:loadincrease}
Let $T'$ be any subtree of $T$, rooted at (an arbitrary) vertex $i_r \in V_{T'}$.
Then the load increase of each edge in $E_{T'} \setminus E_{T'_q}$ resulting from
using at prediction time the circuits contained in $\scC(T'_q,T',G,\rho)$ 
is $\scO(\rho)$.
\end{lemma}
\begin{proof}
Fix a sheaf $E'$ and any circuit $C \in \scC(T'_q,T', G, \rho)$ 
containing the unique queried edge of $E'$, and $C'$ be the part of $C$ that belongs to $T' \setminus T'_q$. 
We know that the edges of $C'$ are potentially loaded by all circuits needed to cover the sheaf, 
which are $\scO(\rho)$. We now check that no more than $\scO(\rho)$ additional circuits use those edges.
Consider the line graph $L$ created by the depth first visit of $T$ starting from $i_r$. 
Each time an edge $(i,j)$ is traversed (even in a backtracking step), the edge is appended to 
$L$, and $j$ becomes the new terminal node of $L$. Hence, each backtracking step generates in $L$ 
at most one duplicate of each edge in $T$, while the nodes in $T$ may be duplicated several times in $L$.
Let $i_{\min}$ and $i_{\max}$ be the nodes of $V_{T'_q} \setminus V_{T'}$ 
incident to the first and the last edge, respectively, 
assigned to sheaf $E'$ during the visit of $T$, where the order is meant to be chronological.
Let $\ell_{\min}$ and $\ell_{\max}$ be the first occurrence of $i_{\min}$ and $i_{\max}$
in $L$, respectively, when traversing $L$ from the first node inserted. 
Let $L'$ be the sub-line of $L$ having $\ell_{\min}$ and $\ell_{\max}$ as terminal nodes. 
By the way \ep\ is defined, all edges of $C'$ that are loaded by circuits covering $E'$ must also 
occur in $L'$. Since each edge of $T$ occurs at most twice in $L$, each edge of $C'$ belongs to 
$L'$ and to at most another sub-line $L''$ of $L$ associated with a different sheaf $E''$. 
Hence the overall load of each edge in $C'$ is $\scO(\rho)$.
\end{proof}

We are now ready to bound the number of mistakes made when \alg\ is run on any labeled graph $(G,Y)$. 
\begin{lemma}\label{l:alg}
The load of each queried edge selected by running $\alg(\rho,\theta)$ 
on any labeled graph $(G = (V_G,E_G),Y)$ 
is $\scO\left(\rho \Bigl(\frac{|E_G|-|V_G|+1}{\theta}+\theta\Bigr)\right)$.
\end{lemma}\label{l:load}
\begin{proof}
In the proof we often refer to the line numbers of the pseudocode in Figure~\ref{f:algsimple}.
We start by observing that the subtrees returned by the calls to \tp\ are disjoint 
(each $T_q$ is removed from $T$ in line~$10$). This entails that \ep\ is called on disjoint subsets of $E_G$, 
which in turn implies that the sheaves returned by calls to \ep\ are also disjoint. Hence, for each training 
edge $(i,j) \in E_G \setminus E_T$ (i.e., selected in line $7$), 
we have $L_{i,j} = \scO(\rho)$. This is because the number of circuits in $\scC(G)$ 
that include $(i,j)$ is equal to the cardinality of the sheaf to which $(i,j)$ belongs 
(lines $8$ and $9$).

We now analyze the load $L_{v,w}$ of each edge $(v,w) \in E_T$. This quantity can be viewed
as the sum of three distinct load contributions: $L_{v,w} = L_{v,w}'+L_{v,w}''+L_{v,w}'''$. 
The first term $L_{v,w}'$ accounts for the load created in line $4$, when $v$ and $w$ belong to 
the same subtree $T_q$ returned by calling \tp\ in line $3$. The other two terms $L_{v,w}''$ 
and $L_{v,w}'''$ take into account the load created in line $9$, when either
$v$ and $w$ both belong ($L_{v,w}''$) or do not belong ($L_{v,w}'''$) 
to the subtree $T_q$ returned in line $3$.

Assume now that both $v$ and $w$ belong to $V_{T_q}$ with $T_q$ returned in line $3$. 
Without loss of generality, let $v$ be the parent of $w$ in $T$. The load contribution $L_{v,w}'$ 
deriving from the circuits in $\scC(G)$ that are meant to cover the test edges joining pairs 
of nodes of $T_q$ (line $4$) must then be bounded by $|E_G(T_w,T_q)|$. This quantity can in turn be 
bounded by $\theta$ using part~(i) of Lemma~\ref{l:tpart}. Hence, we must have $L_{v,w}' \le \theta$.

Observe now that $L_{v,w}''$ may increase by one each time line $9$ is executed. 
This is at most $\scO(\rho) \times |E_G(T_w,T_q)|$. 
Since $|E_G(T_w,T_q)| \le \theta$ by part (i) of Lemma~\ref{l:tpart}, 
we must have $L_{v,w}'' = \scO(\rho\theta)$.

We finally bound the load contribution $L_{v,w}'''$. As we said, this refers to the load created 
in line~$9$ when neither $v$ nor $w$ belong to subtree $T_q$ returned in line $3$. 
Lemma~\ref{l:loadincrease} ensures that, for each call of \ep, $L_{v,w}'''$  
gets increased by $\scO(\rho)$. We then bound the number of times when \ep\ may be called. 
Observe that $|E_G(T_q,T)| \ge \theta$ for each subtree $T_q$ 
returned by $\tp$. Hence, because different calls to \ep\ operate on disjoint edge subsets, 
the number of calls to \ep\ must be bounded by the number of calls to \tp. 
The latter cannot be larger than $\tfrac{|E_G|-|V_G|+1}{\theta}$ which, in turn, implies 
$L_{v,w}''' = \scO\Bigl(\tfrac{\rho}{\theta}(|E_G|-|V_G|+1)\Bigr)$.

Combining together, we find that 
\[
L_{v,w} = L_{v,w}'+L_{v,w}''+L_{v,w}''' 
= 
\scO\Bigl(\theta + \rho\theta + \frac{\rho}{\theta}(|E_G|-|V_G|+1)\Bigr)
\] 
thereby concluding the proof.
\end{proof}
The value of threshold $\theta$ that minimizes the above upper bound is $\theta = \sqrt{|E_G|-|V_G|+1}$,
as exploited next.
\begin{lemma}\label{t:alg}
The number of mistakes made by $\alg(\rho,\sqrt{|E_G|-|V_G|+1})$ on a labeled graph $(G= (V_G,E_G),Y)$ 
is $\scO\Bigl(\Delta_2(Y)\,\rho\,\sqrt{|E_G|-|V_G|+1}\Bigr)$, while we have
$\frac{|\scC(G)|}{Q - |V_G|+1} \ge \rho$, where $Q$ is the size
of the chosen query set (excluding the initial $|V_G|-1$ labels), 
and $|\scC(G)|$ is the size of the test set.
\end{lemma}
\begin{proof}
The condition $\frac{|\scC(G)|}{Q - |V_G|+1} \ge \rho$ 
immediately follows from:
\begin{itemize}
\item [(i)] the very definition of \ep, which selects one queried edge per sheaf,
the cardinality of each sheaf being not smaller than $\rho+1$, and 
\item [(ii)] the very definition of \alg, which queries the labels of 
$|V_G|-1$ edges when drawing the initial spanning tree $T$ of $G$.
\end{itemize}
As for the mistake bound, recall that for each queried edge $(i,j)$, 
the load $L_{i,j}$ of $(i,j)$ is defined to be the number of circuits of $\scC(G)$ 
that include $(i,j)$. As already pointed out, each $\delta$-edge cannot yield more mistakes
than its load. Hence the claim simply follows from 
Lemma \ref{l:alg}, and the chosen value of $\theta$.
\end{proof}

\begin{proof}[Theorem \ref{t:algb}]
In order to prove the condition $\frac{|\scC(G)|}{Q} \ge \frac{\rho-3}{3}$, 
it suffices to consider that:
\begin{itemize}
\item [(i)] a spanning forest
containing at most $|V_G|-1$ queried edges is drawn at each 
execution of the do-while loop in Figure~\ref{f:algfinal},
and
\item [(ii)] \ep\  queries one edge per sheaf,
where the size of each sheaf is not smaller than $\rho+1$.
Because  $\frac{|E_G|}{\rho|V_G|}+1$ bounds the number of do-while loop executions,
we have that the number of queried edges is bounded by
\begin{align*}
(|V_G|-1) \Bigl(\frac{|E_G|}{\rho|V_G|}+1\Bigr) 
&+ 
\frac{1}{\rho+1}\Bigl(|E_G|-(|V_G|-1) \Bigl(\frac{|E_G|}{\rho|V_G|}+1\Bigr)\Bigr)\\
&\le 
\frac{|E_G|}{\rho}+|V_G|+\frac{|E_G|}{\rho}\Bigl(1-\frac{|V_G|-1}{\rho|V_G|}-\frac{|V_G|-1}{|E_G|}\Bigr)\\ 
&\le
|V_G|+2\frac{|E_G|}{\rho} \\
&\le 
3 \frac{|E_G|}{\rho}
\end{align*}
where in the last inequality we used $\rho \leq \frac{|E_G|}{|V_G|}$.
Hence $\frac{|\scC(G)|}{Q} = \frac{|E_G|-Q}{Q} \geq \frac{\rho-3}{3}$, as claimed.
\end{itemize}
Let us now turn to the mistake bound.
Let $V_{G''}$ and $E_{G''}$ denote the node and edge sets of the connected components $G''$ 
of each subgraph $G' \subseteq G$ on which $\algb(\rho)$ invokes $\alg(\rho,\sqrt{E'})$.
By Lemma\ \ref{l:alg}, we know that the load of each queried edge selected
by \alg\ on $G''$ is bounded by 
\begin{align*}
\scO\biggl(\rho \biggl(\frac{|E_{G''}|-|V_{G''}|+1}{\sqrt{E'}}+\sqrt{E'}\biggr)\biggr)
& = \scO( \rho \sqrt{|E'|})
= \scO\Bigl( \rho^{\frac{3}{2}} \sqrt{|V_G|}\Bigr)
\end{align*}
the first equality deriving from 
$|E_{G''}|-|V_{G''}|+1 \le |E'|$, and the second one from 
$|E'| \le \rho|V_G|$ (line 3 of \algb's pseudocode).

In order to conclude the proof, we again use the fact that any $\delta$-edge 
cannot originate more mistakes than its load, along with the observation that
the edge sets of the connected component $G''$
on which \alg\ is run are pairwise disjoint.
\end{proof}

\end{document}